\documentclass[final]{hld2026}

\usepackage{booktabs}
\usepackage{subcaption}

\graphicspath{{../preliminary/}}

\title[Reward-Aware Population Scaling of ES]{Reward-Aware Population Scaling of \\ Evolutionary Strategies in LLM Fine-Tuning}

\hldauthor{%
\Name{Sung Cho} \Email{sungcho9@engineering.upenn.edu}\\
\Name{Gyubin Han} \Email{gyubinh@engineering.upenn.edu}\\
\addr University of Pennsylvania}

\newcommand{\E}{\mathbb{E}}
\newcommand{\Prob}{\mathbb{P}}
\newcommand{\R}{\mathbb{R}}
\newcommand{\1}{\mathbf{1}}
\newcommand{\ce}{\mathrm{CE}}
\newcommand{\acc}{\mathrm{acc}}
\newcommand{\KN}{K_N}
\newcommand{\Navail}{N_{\mathrm{avail}}}
\newcommand{\gradraw}{\hat g_{\mathrm{raw}}}
\newcommand{\gradnorm}{\hat g_{\mathrm{norm}}}

\begin{document}

\maketitle

\begin{abstract}
Using Evolutionary Strategies (ES) for fine-tuning large language models is attractive because it is memory-efficient, parallel, and compatible with black-box or discrete rewards. Yet its population-size conclusions conflict sharply: fine-tuning with cross-entropy (CE) reward succeeds with $N=1$~\citep{malladi2023mezo}, while binary-reward training often needs $N \approx 30$~\citep{qiu2025esscale}. We show this gap is largely about reward design and normalization, not population size. In the capable-model regime we study, z-score advantage normalization can cause $N=2$ to fail. Disabling normalization lets binary-reward ES with $N=2$ improve on GSM8K and TREC across capable models spanning 0.5B--7B, where the normalized variant collapses or degrades. This small-$N$ risk is set by reward granularity: binary accuracy reward induces a zero-advantage probability $q$ that depends in closed form on base accuracy, batch size, and intra-pair correctness correlation; a zero-training probe on Qwen2.5-Instruct/GSM8K matches the formula with mean absolute error 0.020 across 12 configurations and finds the availability threshold $N_{\mathrm{avail}}$ to be small in this capable-model regime. The implication is not that $N=2$ is universally sufficient, but that small-population failure in capable-model binary ES can be an implementation artifact rather than an intrinsic population limit.
\end{abstract}

\section{Introduction}
Zeroth-order and ES methods are appealing for LLM fine-tuning because they avoid backpropagation through long sequences, parallelize across perturbations, and accept black-box or discrete rewards. Yet even the most essential hyperparameter, population size $N$, looks contradictory across prior work. Malladi et al. showed ES fine-tuning can work with one antithetic pair (i.e., $N=1$) with CE reward function~\citep{malladi2023mezo}, whereas Qiu et al. claimed binary-reward ES often appears to need much larger populations (i.e., $N\approx 30$)~\citep{qiu2025esscale}. We argue this is not a paradox but a \emph{scaling law}: $N$ is the scaling variable, and reward granularity together with normalization choices fix where on the law a given setup sits.

$N$ controls how many reward-carrying perturbation directions are available per step. Below an availability threshold $\Navail$, almost no pair carries a usable signal and training stalls; reward granularity sets where $\Navail$ lies, and advantage normalization can shift the threshold by erasing reward-scale information at small $N$.

Reward sparsity is not unique to ES; first-order LLM reinforcement learning (RL) has motivated dense process and token-level reward schemes precisely because terminal binary supervision is weak~\citep{lightman2023verify,wang2023mathshepherd,luo2024omegaprm,cui2025prime,chan2024dense,yoon2024tlcr}. ES exposes a sharper form: because the estimator sees only scalar reward differences, quantization creates literal ties.

Characterizing \emph{when} small populations are at risk (the availability frame) and isolating \emph{why} they then fail (normalization, not population size) are distinct steps: the first sets the stage, and the second is the core result it reveals. We make two contributions:
\begin{enumerate}
  \item \textbf{Reward granularity determines the availability threshold.} CE has $q=0$ (trivial limit, $\Navail=1$); binary has $q>0$ with $\Navail \sim \log\delta/\log q$. We give an explicit binary approximation and validate it with a zero-training probe on Qwen2.5-Instruct~\citep{qwen2.5} (0.5B/1.5B/7B) on GSM8K~\citep{cobbe2021gsm8k} (mean absolute error $0.020$ across 12 configurations).
  \item \textbf{Normalization causes small-$N$ failure.} z-score advantage normalization erases reward-scale information at small $N$ (exactly so at $N=2$). Across Qwen2.5-\{0.5B, 1.5B, 7B\} on GSM8K and TREC~\citep{liroth2002,hovy2001}, simply disabling normalization recovers a binary-reward $N=2$ run that otherwise collapses, so the apparent need for large populations is, in this regime, the implementation rather than the population.
\end{enumerate}

Our claims are intentionally narrow. We do not claim $N=2$ always works, nor that reward degeneracy alone explains every observed $N\approx 30$ threshold. The point is the joint mechanism of reward sparsity together with normalization read inside a scaling-law frame.

\section{Setup: ES as Smoothed High-Dimensional Dynamics}
\paragraph{Notation.}
Let $\theta \in \R^d$ be model parameters, $\mathcal{B}$ a batch of size $B$, $N$ the number of antithetic perturbation pairs, $\sigma > 0$ the perturbation scale, and $\varepsilon_i \sim \mathcal{N}(0,I_d)$. The pair-level reward difference is $A_i = R(\theta+\sigma\varepsilon_i,\mathcal{B}) - R(\theta-\sigma\varepsilon_i,\mathcal{B})$.

\paragraph{Raw and normalized dynamics.}
The raw ES estimator is
\begin{equation}
  \gradraw = \frac{1}{N\sigma}\sum_{i=1}^{N} A_i\,\varepsilon_i,
  \label{eq:graw}
\end{equation}
and the normalized estimator divides the centered advantages by their empirical standard deviation:
\begin{equation}
  \gradnorm = \frac{1}{N\sigma}\sum_{i=1}^{N} \frac{A_i-\bar A}{s_A}\,\varepsilon_i,
  \qquad
  \bar A = \frac{1}{N}\sum_{i=1}^{N} A_i,
  \label{eq:gnorm}
\end{equation}
where $s_A$ is the empirical standard deviation of $\{A_i\}_{i=1}^{N}$. Our experiments use the pair-level convention.

\paragraph{Rewards and smoothing.}
We compare two rewards:
\[
  R_\acc = \tfrac{1}{B} \textstyle\sum_j \1[\hat y_\theta(x_j) = y_j],
  \qquad
  R_\ce = \frac{1}{B} \textstyle\sum_j \log p_\theta(y_j \mid x_j),
\]
where $R_\acc$ is quantized in steps of $1/B$ and $R_\ce$ is continuous in the logits and requires white-box access. Earlier black-box prompt-tuning work observed a similar sparsity gap in derivative-free prompt search~\citep{sun2022bbt}; we extend it from prompt space to full-parameter ES.

ES optimizes the Gaussian-smoothed objective $f_\sigma(\theta)=\E_{\varepsilon \sim \mathcal{N}(0,I_d)}[f(\theta+\sigma\varepsilon)]$, which is $L_\sigma$-smooth with $L_\sigma = 2/\sigma^2$ even when $f$ is discontinuous (as for binary accuracy reward)~\citep{nesterov2017random}. This worst-case $L_\sigma$ appears to forbid the empirically used learning rates ($\eta \sim 10^{-3}$--$10^{-4}$), a gap that spectral decay of the Hessian can account for via average rather than worst-case curvature (Appendix~\ref{app:spectral}). Once the update is divided by the random statistic $s_A$, the estimator changes qualitatively.

\section{A Reward-Aware Population Scaling Law}
\label{sec:scaling}
We treat $N$ as a scaling variable controlling how many reward-carrying perturbation directions are available per ES update, with an availability threshold $\Navail$ below which the update typically contains no reward-carrying direction and training fails. The position of the threshold is set by reward granularity: $\Navail=1$ for dense (CE) reward (Proposition~\ref{prop:ce}); nontrivial for sparse (binary) reward (Proposition~\ref{prop:binary}).

\paragraph{Availability.}
The natural state variable is not just $N$, but the number of perturbation directions that survive degeneracy. Writing $q(\theta,\mathcal{B},\sigma)=\Prob(A=0)$ for the seed-level zero-advantage probability, we have
\begin{equation}
  \KN = \sum_{i=1}^{N}\1[A_i \neq 0],\qquad
  \Prob(\KN=0)=q^N,\qquad
  \Navail(\delta)=\left\lceil \frac{\log \delta}{\log q} \right\rceil .
  \label{eq:navail}
\end{equation}
Under the iid-seed approximation, $\E[\KN]=N(1-q)$. This turns population size into an availability question before it becomes a variance question: if $\KN=0$, the update contains no reward-carrying direction at all, and $\Navail(\delta)$ is the smallest population for which at least one non-degenerate seed appears with probability at least $1-\delta$.

\paragraph{Dense limit and $N$-cancellation.}
When $q=0$ the availability threshold collapses to $\Navail=1$, so any $N\ge 1$ is above threshold. There, the cumulative drift and diffusion across $T$ steps become $N$-independent under the fixed-budget learning-rate scaling.
\begin{proposition}[CE population indifference]
\label{prop:ce}
Let $G_i = \sigma^{-1}\bigl(R(\theta+\sigma\varepsilon_i)-R(\theta-\sigma\varepsilon_i)\bigr)\varepsilon_i$ and write the raw estimator \eqref{eq:graw} as $\hat g_N = N^{-1}\sum_{i=1}^N G_i$ to make the population dependence explicit. Under any reward continuous in $\theta$ (in particular CE),
\[
  \E[\hat g_N] = 2\,\nabla f_\sigma(\theta),
  \qquad
  \mathrm{Cov}(\hat g_N) = \Sigma(\theta)/N,
\]
with $\Sigma(\theta) = \mathrm{Cov}(G_i)$. Consequently, at fixed total perturbation-pair budget $M = NT$ and learning-rate scaling $\eta_N = N\eta_0$, the cumulative drift and diffusion across $T$ steps,
\[
  \E[\Delta\theta_N] = 2\,M\eta_0\,\nabla f_\sigma(\theta),
  \qquad
  \mathrm{Cov}(\Delta\theta_N) = M\eta_0^2\,\Sigma(\theta),
\]
are both independent of $N$ in the local linearization around $\theta$.
\end{proposition}
The proof is in Appendix~\ref{app:ce-proof}.

\begin{corollary}[No availability threshold under CE]
\label{cor:ce-nondegen}
Under standard smooth-network assumptions (Appendix~\ref{app:ce-proof}), $\KN = N$ almost surely whenever $\nabla_\theta R_{\ce}(\theta,\mathcal{B}) \neq 0$, and Proposition~\ref{prop:ce} applies seedwise; see Appendix~\ref{app:pop-scaling} for matching CE-reward scaling curves.
\end{corollary}

\paragraph{Sparse regime.}
Binary reward makes $q>0$, so the threshold is nontrivial and $\KN$ becomes a genuine random count rather than identically $N$.
\begin{proposition}[Binary degeneracy approximation]
\label{prop:binary}
Under a small-perturbation, homogeneous-batch approximation,
\begin{equation}
  q=\Prob(A_{\acc}=0)\approx
  \frac{1}{\sqrt{4\pi B\,p_0(1-p_0)(1-\rho)}},
  \label{eq:qapprox}
\end{equation}
where $p_0$ is the base accuracy and $\rho$ is the correlation between the two perturbed models' per-example correctness---how often $\theta+\sigma\varepsilon$ and $\theta-\sigma\varepsilon$ get the same items right or wrong.
\end{proposition}
The right-hand side is a local-CLT point-density approximation at zero (Appendix~\ref{app:binary-proof}); in the extreme sparse regime ($B$ small or $p_0$ near $\{0,1\}$) it can formally exceed $1$, in which case we use $\min\{1,\cdot\}$ for numerical work and prefer the empirical probe (Appendix~\ref{app:probe}) for calibration. Above $\Navail$, the effective non-degenerate count $\KN \to N(1-q)$ in expectation and population averaging behaves in the usual way; below $\Navail$, the random count $\KN$ itself is the bottleneck and no learning-rate rescaling can recover a gradient-carrying update from a step where $\KN=0$. Appendix~\ref{app:ce-proof} pinpoints where the CE argument breaks in this regime.

\begin{table}[t]
  \centering
  \scriptsize
  \caption{Full pre-training degeneracy probe on GSM8K at $\sigma=10^{-3}$ with $K=200$ perturbation pairs. $|\Delta q| = |\text{empirical }q - \text{predicted }q|$; mean absolute error across the 12 rows is $0.020$. $\Navail$ is computed from empirical $\hat q$ via \eqref{eq:navail} ($\delta=0.05$); applying the same formula to the predicted $q$ column disagrees on three rows. All models are the official Qwen2.5-\{0.5B, 1.5B, 7B\}-Instruct checkpoints.}
  \label{tab:probe-results}
  \begin{tabular}{lccccccc}
    \toprule
    Model & $B$ & $p_0$ & $\hat\rho$ & empirical $q$ & predicted $q$ & $|\Delta q|$ & $\Navail$ \\
    \midrule
    Qwen2.5-0.5B & 4  & 0.306 & 0.357 & 0.345 & 0.382 & 0.037 & 3 \\
    Qwen2.5-0.5B & 8  & 0.306 & 0.379 & 0.240 & 0.275 & 0.035 & 3 \\
    Qwen2.5-0.5B & 16 & 0.308 & 0.370 & 0.165 & 0.192 & 0.027 & 2 \\
    Qwen2.5-0.5B & 32 & 0.308 & 0.362 & 0.140 & 0.135 & 0.005 & 2 \\
    \midrule
    Qwen2.5-1.5B & 4  & 0.466 & 0.439 & 0.400 & 0.377 & 0.023 & 4 \\
    Qwen2.5-1.5B & 8  & 0.466 & 0.475 & 0.290 & 0.276 & 0.014 & 3 \\
    Qwen2.5-1.5B & 16 & 0.466 & 0.424 & 0.190 & 0.186 & 0.004 & 2 \\
    Qwen2.5-1.5B & 32 & 0.466 & 0.449 & 0.125 & 0.135 & 0.010 & 2 \\
    \midrule
    Qwen2.5-7B & 4  & 0.642 & 0.607 & 0.495 & 0.469 & 0.026 & 5 \\
    Qwen2.5-7B & 8  & 0.642 & 0.622 & 0.340 & 0.338 & 0.002 & 3 \\
    Qwen2.5-7B & 16 & 0.642 & 0.625 & 0.220 & 0.240 & 0.020 & 2 \\
    Qwen2.5-7B & 32 & 0.642 & 0.612 & 0.130 & 0.167 & 0.037 & 2 \\
    \bottomrule
  \end{tabular}
\end{table}

\paragraph{Empirical probe.}
Throughout, all Qwen2.5 models are the official instruction-tuned checkpoints; we abbreviate sizes as Qwen2.5-0.5B/\allowbreak1.5B/\allowbreak7B. A zero-training probe on Qwen2.5-Instruct (0.5B/\allowbreak1.5B/\allowbreak7B) on GSM8K validates \eqref{eq:qapprox} with MAE $0.020$ across 12 configurations (Table~\ref{tab:probe-results}; probe protocol in Appendix~\ref{app:probe}). For Qwen2.5-1.5B at $B=16$, $\hat q \approx 0.19$, giving $\Prob(\KN=0)\approx 0.036$ at $N=2$---so the availability threshold is small here, yet normalized binary-reward ES still fails at $N=2$. The remaining mechanism is \S\ref{sec:norm}. This section characterizes \emph{when} small $N$ is at risk (availability), while \S\ref{sec:norm} explains \emph{why} it then fails (normalization).

\section{Normalization Distorts the Threshold}
\label{sec:norm}
Since $\Prob(\KN=0) \approx 0.036$ at $N=2$ for Qwen2.5-1.5B, pure availability cannot explain the $N=2$ collapse---the remaining mechanism is advantage normalization. Raw ES is \emph{self-annealing}: if reward differences shrink near degeneracy, the update shrinks with them. Z-score normalization replaces that magnitude with a small-$N$ scale estimate computed from a handful of sparse advantages, removing the self-annealing and providing an ES-specific instance of how scale-uncontrolled training signals can create new failure modes~\citep{gao2025effective}.

\begin{proposition}[Two-sample z-score erases advantage scale]
\label{prop:zscore}
Let $A_1 \neq A_2$ be pair-level advantages, let $\bar A = (A_1+A_2)/2$, and let $s_A$ be their empirical standard deviation. Then the normalized vector
\[
  \left(\frac{A_1-\bar A}{s_A},\frac{A_2-\bar A}{s_A}\right)
\]
is independent of the absolute gap $|A_1-A_2|$ up to a convention-dependent constant. With population standard deviation it is $(\pm 1,\mp 1)$; with sample standard deviation it is $(\pm 1/\sqrt{2},\mp 1/\sqrt{2})$.
\end{proposition}
The proof is in Appendix~\ref{app:norm-analysis}. Z-score normalization always removes absolute scale, but at $N=2$ even ratio information disappears: an infinitesimal advantage gap and a large advantage gap therefore produce the same normalized update magnitude.

Under sparse binary rewards this matters most. Many $A_i$ are exactly zero, and the nonzero ones can be very small near degeneracy: with raw advantages such steps naturally decay, but normalization promotes the one or two surviving nonzero seeds to fixed-size updates while estimating $s_A$ from only a handful of sparse rewards. Normalization does not merely standardize the signal---it changes the stochastic dynamics.

\begin{figure}[t]
  \centering
  \includegraphics[width=\linewidth]{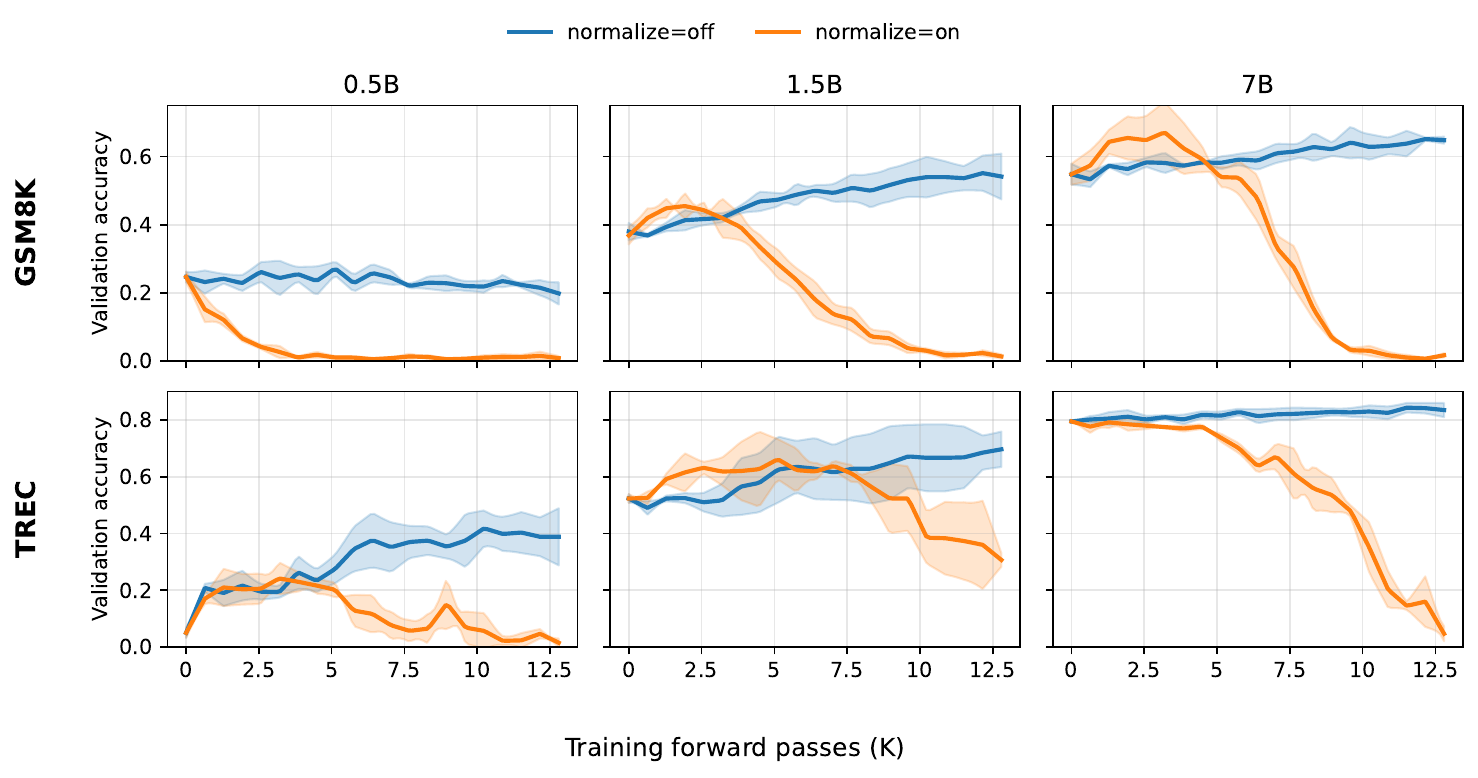}
  \caption{\textbf{Turning off advantage normalization recovers small-population ES across model sizes and tasks.}
  Advantage-normalization ablation at $N=2$ with binary reward (seeds 42/43/44; mean $\pm$ std),
  $\sigma=10^{-3}$, $\eta=10^{-3}$, $B=16$, 200 iterations;
  for Qwen2.5-\{0.5B, 1.5B, 7B\} (columns) on GSM8K (top row) and TREC (bottom row).
  In every panel, z-score normalization (\emph{on}, orange) collapses (i.e. accuracy degrades toward zero) %
  whereas raw advantages (\emph{off}, blue) preserve reward-scale information.
  The upward learning under raw advantages appears once the base model is capable enough for the task:
  it holds for 1.5B/7B on both tasks and for 0.5B on TREC, while 0.5B on GSM8K only avoids collapse
  without improving (its base accuracy too low for GSM8K reasoning).
  Thus the \emph{on}-collapse is robust across sizes and tasks, whereas whether \emph{off} improves
  (rather than merely holding) depends on base-model capability.}
  \label{fig:norm_ablation}
  \vspace{-0.75em}
\end{figure}

The binary population-scaling curves in Appendix~\ref{app:pop-scaling} use the default normalized implementation; Figure~\ref{fig:norm_ablation} is a single fixed-everything ablation that varies only the treatment of advantages, now spanning Qwen2.5-\{0.5B, 1.5B, 7B\} on GSM8K and TREC. With normalization on, the $N=2$ run collapses toward zero accuracy in every size/task cell---on the capable models it briefly rises before collapsing. With normalization off, the same binary-reward ES setup avoids collapse everywhere, improving steadily to, e.g., roughly $0.55$ validation accuracy on 1.5B/GSM8K. The two-part reading is that the normalization-on collapse is robust to model size, whereas the upward improvement under raw advantages depends on base capability: it appears wherever the base model is already competent at the task (1.5B/7B on both tasks, 0.5B on TREC) and is absent only for 0.5B on GSM8K, where raw ES merely avoids collapse without improving. We do not claim $N=2$ is universally sufficient or that normalization explains every small-$N$ failure, but in these settings, the apparent $N=2$ collapse is the implementation, not the population size. Within the scaling-law frame, normalization is the implementation knob that decides whether $N$ near $\Navail$ behaves like the raw self-annealing regime or like a fixed-magnitude random walk that ignores the very scale information $\Navail$ was defined to track.

\section{Discussion and Limitations}
Practical recommendation: estimate $\Navail$ via the zero-training probe (Appendix~\ref{app:probe}) and start with the smallest $N$ that clears it. In the Qwen2.5-Instruct/GSM8K regime we measured, $\Navail \leq 5$ across all 12 configurations and $N=2$ with raw advantages was already viable. Tune upward only if optimizer stability, normalization, or wall-clock parallelism require it; near $\Navail$, raw or clipped-std advantages are safer than z-scoring. This complements curvature-based ES analyses~\citep{liang2026blessing}: geometry determines whether useful directions exist; reward sparsity and normalization determine whether scalar feedback reveals them.

Three caveats. (a) The normalization-off recovery of $N=2$ is consistent across three model sizes and two tasks (Qwen2.5-\{0.5B, 1.5B, 7B\} on GSM8K and TREC): the normalization-on collapse appears in every case, whereas whether raw ES \emph{improves} (rather than merely avoiding collapse) depends on base-model capability---so we still do not claim $N=2$ is universally sufficient. (b) The degeneracy model is an availability bound under homogeneous-batch, small-perturbation assumptions; \eqref{eq:qapprox} is an approximation, and we treat $\rho$ as a complementary diagnostic only---it indexes the variance of the reward difference but does not predict whether the surviving differences align with useful parameter directions. (c) Proposition~\ref{prop:ce}'s drift--diffusion invariance is a moment-level statement under local linearization at fixed $\theta$; we do not derive a descent-level $N$-cancellation result and make no claim about whether different $N$ values produce equivalent full training trajectories at fixed compute.

\acks{This work originated as a course project for STAT 4830 at the University of Pennsylvania; we thank Damek Davis for his guidance. Compute credits were provided by Prime Intellect.}

\newpage
\bibliography{sample}

\clearpage
\appendix

\section{Algorithmic Conventions: MeZO and ES-at-Scale}
\label{app:conventions}

\paragraph{Estimator normalization.}
The two paradigms our paper engages with differ only in the constant in front of the antithetic difference~\citep{salimans2017openai,malladi2023mezo}. MeZO uses
\begin{equation}
  \hat g_{\mathrm{MeZO}} = \frac{R(\theta+\sigma\varepsilon) - R(\theta-\sigma\varepsilon)}{2\sigma}\,\varepsilon,
  \label{eq:mezo}
\end{equation}
which satisfies $\E[\hat g_{\mathrm{MeZO}}] = \nabla f_\sigma(\theta)$ directly via the antithetic Stein identity (Appendix~\ref{app:smoothing}). ES-at-scale \citep{qiu2025esscale} uses the $1/\sigma$ form of \eqref{eq:graw} instead, giving $\E[\hat g_N] = 2\,\nabla f_\sigma(\theta)$ with the factor of two absorbed into the learning rate. The two are equivalent at the level of the expected update once learning rates are matched: at literal hyperparameter $\eta$, the ES estimator moves twice as far per step as MeZO, so ES's effective rate is $2\eta$ relative to MeZO. We follow the $1/\sigma$ convention throughout to align with \citet{qiu2025esscale}; Proposition~\ref{prop:ce} and Appendix~\ref{app:smoothing} are stated under the same convention.

\paragraph{Which quantities are normalized.}
Our analysis assumes the implementation normalizes the pair-level advantages $A_i=R(\theta+\sigma\varepsilon_i)-R(\theta-\sigma\varepsilon_i)$ across $i=1,\dots,N$. Under this convention, the $N=2$ scale-erasure claim of Proposition~\ref{prop:zscore} follows from two-sample z-scoring of $(A_1,A_2)$; it does \emph{not} require assuming that the algorithm separately normalizes the positive and negative rollouts or that the pair-level advantages are exactly $\{+a,-a\}$. Some MeZO and ES implementations optionally divide $A_i$ by the mini-batch standard deviation; whether this helps depends on the regime, as Section~\ref{sec:norm} shows for binary reward.

\paragraph{Conceptual difference between the two paradigms.}
The distinction between MeZO and ES-at-scale is not which prefactor is ``correct''---both produce unbiased estimates of $\nabla f_\sigma$ up to a constant factor. The substantive difference is the reward signal:
\begin{itemize}
  \item \textbf{MeZO} uses CE reward, which is white-box and incurs no availability threshold (Corollary~\ref{cor:ce-nondegen}); $N=1$ is sufficient.
  \item \textbf{ES-at-scale} uses binary accuracy reward, which is black-box compatible but sparse (Proposition~\ref{prop:binary}); the reported $N \approx 30$ is consistent with either mechanism in our account---a substantial availability threshold in their base-accuracy regime, advantage normalization (Section~\ref{sec:norm}), or both---and, having measured neither directly in their setting, we do not assert which dominates there; in our own Qwen2.5/GSM8K probe availability turned out small, leaving normalization as the operative factor.
\end{itemize}
Both choices are essentially optimal for their setting. The $N \approx 30$ floor may reflect the interaction of binary reward with the base-accuracy regime of their structured-prompt evaluation---with normalization choices plausibly contributing as much---rather than a fundamental property of the LLM or task; in our Qwen2.5/GSM8K probe (Appendix~\ref{app:probe}), the same mechanism gives a much smaller $\Navail$.

\paragraph{Clipped-standard-deviation normalization.}
When we refer to a possible remedy, we mean replacing \eqref{eq:gnorm} by
\[
  \tilde A_i = \frac{A_i-\bar A}{\max(s_A,\varepsilon_0)}
\]
for a floor $\varepsilon_0>0$. We do not apply this clipped rule in our experiments (equivalently $\varepsilon_0=0$); it is a proposed intervention rather than a validated result here.

\section{Gaussian Smoothing, Stein Identity, and \texorpdfstring{$L_\sigma$}{Lsigma}-Smoothness}
\label{app:smoothing}

\paragraph{Definition and approximation gap.}
Let $f:\R^d \to \R$ be the reward objective and define $f_\sigma(\theta) = \E_\varepsilon[f(\theta+\sigma\varepsilon)]$ with $\varepsilon \sim \mathcal{N}(0, I_d)$. For $M$-Lipschitz $f$, $|f_\sigma(\theta) - f(\theta)| \leq M\sigma$; at $\sigma = 10^{-3}$ this gap is negligible compared to initial suboptimality.

\paragraph{Stein gradient identity.}
By Gaussian integration by parts, $\nabla f_\sigma(\theta) = \sigma^{-1}\,\E[f(\theta+\sigma\varepsilon)\,\varepsilon]$. For the antithetic pair, $\E[(f(\theta+\sigma\varepsilon) - f(\theta-\sigma\varepsilon))\,\varepsilon] = 2\sigma\,\nabla f_\sigma(\theta)$. The $1/(2\sigma)$-prefactor estimator (e.g., MeZO) is unbiased for $\nabla f_\sigma$ exactly; the $1/\sigma$ convention $\gradraw$ used in the main text differs by a factor of two absorbed into the learning rate. Normalizing by the data-dependent statistic $s_A$ changes the estimator qualitatively and removes this direct unbiasedness interpretation.

\paragraph{$L_\sigma$-smoothness.}
\begin{theorem}[\citealt{nesterov2017random}]
\label{thm:smooth}
If $|f| \leq 1$, then $f_\sigma$ is $L_\sigma$-smooth with $L_\sigma = 2/\sigma^2$.
\end{theorem}

\begin{proof}
The second-order Stein identity gives $[\nabla^2 f_\sigma(\theta)]_{kl} = \sigma^{-2}\,\E_\varepsilon[f(\theta+\sigma\varepsilon)(\varepsilon_k\varepsilon_l - \delta_{kl})]$. For any unit vector $v\in\R^d$,
\[
  v^\top \nabla^2 f_\sigma\, v = \frac{1}{\sigma^2}\,\E_\varepsilon\!\left[f(\theta+\sigma\varepsilon)\bigl((v\cdot\varepsilon)^2 - 1\bigr)\right].
\]
Bounding with $|f| \leq 1$ and $\E[(v\cdot\varepsilon)^2] = 1$ gives $|v^\top \nabla^2 f_\sigma\, v| \leq \sigma^{-2}\E[(v\cdot\varepsilon)^2 + 1] = 2/\sigma^2$. This holds for all unit $v$, so $\|\nabla^2 f_\sigma\|_\mathrm{op} \leq 2/\sigma^2 = L_\sigma$.
\end{proof}

The crucial feature is that no smoothness assumption on $f$ is required: binary accuracy reward is a step function ($L = \infty$ classically) yet $f_\sigma$ has finite smoothness. The descent lemma for $f_\sigma$ is therefore exact:
\[
  \E[f_\sigma(\theta + \eta\hat g)] - f_\sigma(\theta) \geq \eta\|\nabla f_\sigma\|^2 - \frac{\eta^2 L_\sigma}{2}\,\E[\|\hat g\|^2].
\]

\begin{remark}[Cocoercivity without PL]
$L_\sigma$-smoothness alone gives $\|\nabla f_\sigma(\theta)\|^2 \leq 2L_\sigma(f_\sigma^* - f_\sigma(\theta))$: a single ascent step at rate $1/L_\sigma$ from $\theta$ improves $f_\sigma$ by $\|\nabla f_\sigma\|^2/(2L_\sigma)$, and this cannot exceed $f_\sigma^* - f_\sigma(\theta)$. No PL condition is needed.
\end{remark}

\section{CE Population Indifference and Non-Degeneracy}
\label{app:ce-proof}

\paragraph{Proof of Proposition~\ref{prop:ce}.}
With $G_i = \sigma^{-1}(R(\theta+\sigma\varepsilon_i)-R(\theta-\sigma\varepsilon_i))\varepsilon_i$ and $\hat g_N = N^{-1}\sum_i G_i$, the antithetic Stein identity from Appendix~\ref{app:smoothing} gives
\[
  \E\!\left[(R(\theta+\sigma\varepsilon)-R(\theta-\sigma\varepsilon))\,\varepsilon\right]
  = 2\sigma\,\nabla f_\sigma(\theta),
\]
hence $\E[G_i] = 2\,\nabla f_\sigma(\theta)$ and $\E[\hat g_N] = 2\,\nabla f_\sigma(\theta)$. The seeds $\varepsilon_i$ are iid, so the $G_i$ are iid with covariance $\Sigma(\theta) = \mathrm{Cov}(G_i)$, and
\[
  \mathrm{Cov}(\hat g_N) = \frac{1}{N}\,\Sigma(\theta).
\]
For the budget argument, fix $\theta$ and consider one ES step at population $N$ with learning rate $\eta_N = N\eta_0$. The single-step expected drift is
\[
  \eta_N\,\E[\hat g_N] = 2\,N\eta_0\,\nabla f_\sigma(\theta),
\]
and the single-step covariance is
\[
  \eta_N^2\,\mathrm{Cov}(\hat g_N) = N\eta_0^2\,\Sigma(\theta).
\]
Across $T = M/N$ independent steps, under the local linearization that $\theta$ does not drift far from the reference point so that $\nabla f_\sigma$ and $\Sigma$ are constant,
\[
  \E[\Delta\theta_N] = T\,\eta_N\,\E[\hat g_N] = 2\,M\eta_0\,\nabla f_\sigma(\theta),
  \qquad
  \mathrm{Cov}(\Delta\theta_N) = T\,\eta_N^2\,\mathrm{Cov}(\hat g_N) = M\eta_0^2\,\Sigma(\theta).
\]
Both quantities are independent of $N$.

\paragraph{Proof of Corollary~\ref{cor:ce-nondegen}.}
Fix a batch $\mathcal{B}$ and define
\[
  h(\varepsilon)=R_{\ce}(\theta+\sigma\varepsilon,\mathcal{B})
  -R_{\ce}(\theta-\sigma\varepsilon,\mathcal{B}).
\]
Under standard transformer architectures---compositions of analytic operations (matrix multiplication, softmax, smooth activations such as GELU/SwiGLU)---$R_{\ce}(\theta,\mathcal{B})$ is real analytic in $\theta$, so $h$ is real analytic in $\varepsilon$. Its gradient at the origin is
\[
  \nabla_\varepsilon h(0) = 2\sigma\,g_{\ce}(\theta,\mathcal{B}),
  \qquad
  g_{\ce}(\theta,\mathcal{B})=\nabla_\theta R_{\ce}(\theta,\mathcal{B})
  = \frac{1}{B}\sum_{j=1}^{B}\nabla_\theta \log p_\theta(y_j\mid x_j).
\]
If $g_{\ce}(\theta,\mathcal{B}) \neq 0$, then $h$ is not identically zero; the zero set of a non-identically-zero real analytic function on $\R^d$ has Lebesgue measure zero, and since the Gaussian measure is absolutely continuous with respect to Lebesgue measure, $\Prob(h(\varepsilon)=0) = 0$. Hence $\KN = \sum_i \1[A_i \neq 0] = N$ almost surely, and the variance-averaging structure used in the proof of Proposition~\ref{prop:ce} is well-defined seedwise.

\paragraph{Where the argument breaks for binary reward.}
The decomposition $\mathrm{Cov}(\hat g_N) = \Sigma/N$ assumed each seed contributes independently to the per-step covariance. Under binary reward, this fails because the effective sample size in the variance averaging is the random count $\KN$, not $N$. When $\KN = 0$, the single-step update is either zero (raw advantages) or normalization-dependent (normalized advantages), and no learning-rate rescaling $\eta_N \propto N$ can recover a gradient-carrying update. This is the precise sense in which CE population indifference is dense-reward-specific.

\section{Binary Degeneracy: Full Proof}
\label{app:binary-proof}

\paragraph{Setup.}
Let $X_j^+, X_j^- \in \{0,1\}$ be the correctness indicators for example $j$ under $\theta\pm\sigma\varepsilon$. Under the homogeneous-batch approximation, $\E[X_j^+] = \E[X_j^-] = p_0$ and $\mathrm{Corr}(X_j^+, X_j^-) = \rho$. Define $Y_j = X_j^+ - X_j^- \in \{-1, 0, 1\}$; then $\E[Y_j] = 0$ and $\mathrm{Var}(Y_j) = 2p_0(1-p_0)(1-\rho) \triangleq v$. The batch-level accuracy difference is proportional to $S_B = \sum_{j=1}^B Y_j$, and $A_\acc = 0$ iff $S_B = 0$.

\paragraph{Local CLT, not standard CLT.}
The standard CLT approximates CDFs, but we need the point probability $\Prob(S_B = 0)$. Gnedenko's local CLT works directly on integer-valued PMFs: for iid integer summands $Y_1, \ldots, Y_B$ with mean $0$ and variance $v$,
\begin{equation}
  \left|\Prob(S_B = k) - \frac{1}{\sigma_S\sqrt{2\pi}}\,\varphi\!\left(\frac{k}{\sigma_S}\right)\right|
  \leq \frac{C\,\E[|Y_j|^3]}{\sigma_S^3},
  \label{eq:local-clt}
\end{equation}
where $\sigma_S^2 = Bv$, $\varphi$ is the standard normal density, and $C > 0$ is a universal constant. Setting $k = 0$ and $\varphi(0) = 1/\sqrt{2\pi}$ gives the leading-order $\Prob(S_B = 0) = (2\pi B v)^{-1/2}$ with a Berry--Esseen error of $O(B^{-1/2})$.

\paragraph{$O(B^{-1})$ error: vanishing third cumulant.}
For our $Y_j \in \{-1, 0, +1\}$ the error is one order better. Since $Y_j^3 = Y_j$ exactly, the third cumulant
\[
  \kappa_3 = \E[Y_j^3] = \E[Y_j] = 0.
\]
The Edgeworth expansion at $k = 0$ reads
\[
  \Prob(S_B = 0) = \frac{1}{\sqrt{2\pi Bv}}\!\left[1 + \frac{\kappa_3}{6 v^{3/2}\sqrt{B}}H_3(0) + \frac{\kappa_4}{24 v^2 B}H_4(0) + O(B^{-3/2})\right],
\]
with $H_3(0) = 0$ and $H_4(0) = 3$. The $O(B^{-1/2})$ Berry--Esseen term proportional to $\kappa_3 H_3(0)$ vanishes identically, leaving the $O(B^{-1})$ fourth-cumulant term as leading. Substituting $\sigma_S^2 = Bv$ yields Proposition~\ref{prop:binary}:
\[
  \Prob(A_\acc = 0) = \frac{1}{\sqrt{4\pi B\,p_0(1-p_0)(1-\rho)}} + O(B^{-1}).
\]

\paragraph{Non-IID heterogeneity.}
For heterogeneous batches with per-example variance $v_j = 2p_j(1-p_j)(1-\rho_j)$, Petrov's non-iid local CLT applies with $\mathrm{Var}(S_\mathrm{het}) = \sum_j v_j = B\bar v$ where $\bar v = B^{-1}\sum_j v_j$. At leading order, $\Prob(S_\mathrm{het}=0) \approx (2\pi B\bar v)^{-1/2}$---the same as Proposition~\ref{prop:binary} with $v$ replaced by $\bar v$; higher-order corrections depend on the empirical distribution of $\{v_j\}$ and we do not characterize them here. The worst case is bimodal accuracy ($p_j$ near $0$ or $1$), where each $v_j \approx 0$ and $\Prob(A=0) \to 1$. The empirical probe (Appendix~\ref{app:probe}) measures $\hat\Prob(A=0)$ directly and is preferred whenever the homogeneous-batch assumption is doubtful.

\section{Pre-Training Degeneracy Probe: Protocol and Results}
\label{app:probe}

The degeneracy probe is a forward-pass-only diagnostic that we use to validate Proposition~\ref{prop:binary} and that practitioners can also run before committing to a population size.

\paragraph{Protocol.}
\begin{enumerate}
  \item Sample $K$ perturbation pairs $\{\varepsilon_i\}_{i=1}^K$ from $\mathcal{N}(0,I_d)$ at the $\sigma$ intended for training.
  \item Evaluate $R(\theta_0+\sigma\varepsilon_i,\mathcal{B})$ and $R(\theta_0-\sigma\varepsilon_i,\mathcal{B})$ on a fixed batch of size $B$. Total cost: $2KB$ forward passes.
  \item Estimate the zero-advantage rate $\hat q = K^{-1}\,|\{i : R(\theta_0+\sigma\varepsilon_i,\mathcal{B}) = R(\theta_0-\sigma\varepsilon_i,\mathcal{B})\}|$, the base accuracy $p_0$, and the intra-pair correctness correlation $\hat\rho$ from the paired indicators.
  \item Compute $\Navail(\delta) = \lceil \log\delta / \log \hat q \rceil$ at the desired confidence (we use $\delta = 0.05$ throughout).
\end{enumerate}

For $K=200$ and $B=16$, the probe costs $6{,}400$ forward passes (on the order of a handful of training iterations) and returns a data-driven population floor that bypasses the homogeneous-batch approximation in Proposition~\ref{prop:binary}. We recommend it as a cheaper alternative to sweeping $N$ empirically: under binary reward, any choice of $N < \Navail$ is dominated by the availability failure mode of Section~\ref{sec:scaling}, irrespective of optimizer or learning-rate tuning. The same probe outputs also yield $\hat\rho$ at no extra forward-pass cost, giving a complementary read on whether the local reward landscape is becoming nearly frozen ($\rho \to 1$).

\paragraph{Results.}
Table~\ref{tab:probe-results} (in the main text, Section~\ref{sec:scaling}) reports the probe applied to Qwen2.5-Instruct (0.5B, 1.5B, 7B) on GSM8K with $K=200$ and $\sigma=10^{-3}$ across $B \in \{4, 8, 16, 32\}$.

\section{Population Scaling Under CE vs Binary Reward}
\label{app:pop-scaling}
Figure~\ref{fig:pop-scaling} below shows population-scaling trajectories under CE versus binary reward. The binary runs all use the default advantage-normalized ES implementation; the normalization ablation is reported separately in Section~\ref{sec:norm}.

\begin{figure}[ht]
  \centering
  \includegraphics[width=\linewidth]{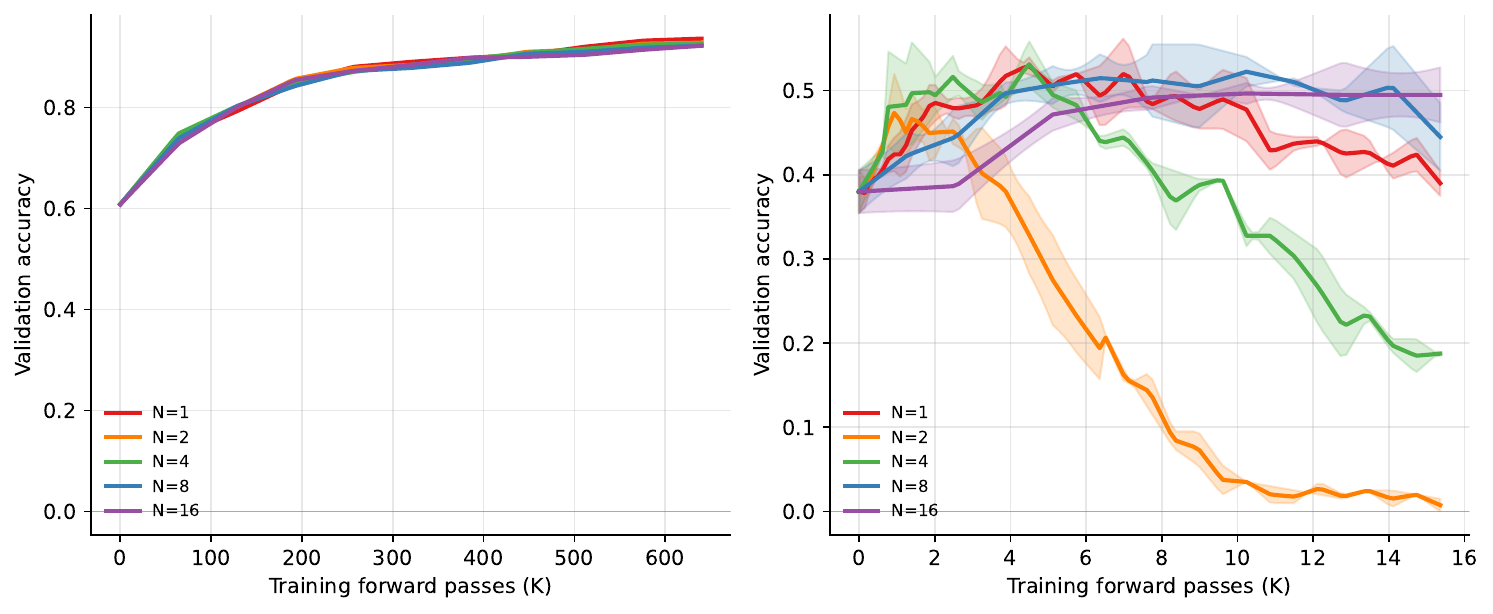}\\[0.2em]
  \begin{minipage}[t]{0.5\linewidth}\centering {\small (a) CE reward on OPT-13B~\citep{zhang2022opt}/SST-2~\citep{socher2013sst}.}\end{minipage}%
  \hfill
  \begin{minipage}[t]{0.5\linewidth}\centering {\small (b) Binary reward on Qwen2.5-1.5B (normalized ES).}\end{minipage}
  \caption{\textbf{Reward choice changes the population-scaling regime.}
  Left: under CE reward with appropriate learning-rate scaling, trajectories are nearly population-indifferent, consistent with CE advantages being nonzero almost surely. Right: under binary reward with the default advantage-normalized ES implementation, trajectories show apparent population dependence ($N=2$ and $N=4$ fail while larger populations are more stable); this reflects the normalized setting rather than an intrinsic population limit---Section~\ref{sec:norm} isolates normalization as the driver (ablation in Figure~\ref{fig:norm_ablation}), so these curves should not be read as failures of raw-advantage ES at the same population sizes.
  Binary-reward curves (right) show mean $\pm$ std over seeds 42/43/44.}
  \label{fig:pop-scaling}
\end{figure}

\section{Advantage Normalization Analysis}
\label{app:norm-analysis}
\paragraph{Proof of Proposition~\ref{prop:zscore}.}
For two numbers $A_1,A_2$, their mean is $\bar A=(A_1+A_2)/2$. Their centered values are
\[
  A_1-\bar A = \frac{A_1-A_2}{2},
  \qquad
  A_2-\bar A = -\frac{A_1-A_2}{2}.
\]
With the population-standard-deviation convention,
\[
  s_A = \sqrt{\frac{(A_1-\bar A)^2+(A_2-\bar A)^2}{2}}
      = \frac{|A_1-A_2|}{2},
\]
so the normalized vector is $(\operatorname{sign}(A_1-A_2),-\operatorname{sign}(A_1-A_2))$. With the sample-standard-deviation convention,
\[
  s_A = \sqrt{(A_1-\bar A)^2+(A_2-\bar A)^2}
      = \frac{|A_1-A_2|}{\sqrt{2}},
\]
so the normalized vector is $(\operatorname{sign}(A_1-A_2)/\sqrt{2},-\operatorname{sign}(A_1-A_2)/\sqrt{2})$. In either case, the absolute gap $|A_1-A_2|$ cancels.

\paragraph{Affine scale invariance for general $N$.}
For any vector $A \in \R^N$, z-score normalization is invariant to positive affine transformations $A \mapsto cA + b\mathbf{1}$. Thus normalization always removes absolute scale. At larger $N$, however, the centered \emph{ratios} among coordinates still survive, which is why the pathology weakens with $N$.

\paragraph{$N=4$ example.}
Suppose the four pair-level advantages are $(0,0,\epsilon,2\epsilon)$ with $\epsilon>0$. After z-score normalization, the result is exactly the same as for $(0,0,1,2)$. Absolute magnitude has disappeared; only the ratio pattern survives. This is less extreme than the $N=2$ case, because the coordinates are not collapsed to a fixed sign vector, but it still removes the self-annealing of raw ES.

\paragraph{Why small $N$ is unstable under sparse binary reward.}
Binary rewards are quantized, so many $A_i$ are exactly zero and the surviving nonzero ones often come from one or two perturbation pairs. In that regime, the empirical standard deviation $s_A$ is estimated from very few support points. The normalized update therefore mixes two effects: it discards absolute reward scale and simultaneously amplifies the noise of a poorly estimated scale statistic. A clipped-standard-deviation rule,
\[
  \tilde A_i = \frac{A_i-\bar A}{\max(s_A,\varepsilon_0)},
\]
is one direct way to prevent this amplification while retaining some of normalization's scale control.

\section{Spectral Decay: Resolving the Theory-Practice Gap}
\label{app:spectral}

This appendix sits outside the main scaling-law argument: it addresses the stability concern raised by the worst-case $L_\sigma$ bound of Theorem~\ref{thm:smooth}, which would forbid the empirically feasible learning rates used throughout this paper. Readers willing to take the empirical $\eta$ regime as given may skip it.

Theorem~\ref{thm:smooth} gives $L_\sigma = 2/\sigma^2 \approx 2 \times 10^6$ at $\sigma = 0.001$, implying stability $\eta < \sigma^2/2 \approx 5 \times 10^{-7}$. Yet ES fine-tuning runs successfully at $\eta \sim 10^{-3}$--$10^{-4}$, a gap of $3$--$4$ orders of magnitude. We trace this to a geometric property of the ES estimator.

\paragraph{Average curvature, not worst-case.}
The ES estimator $\hat{g} = (N\sigma)^{-1}\sum_i A_i\varepsilon_i$ is a \emph{random isotropic direction} in $\R^d$: the perturbations $\varepsilon_i$ have no preferred direction. Therefore the curvature experienced per ES step is
\begin{equation}
  \frac{\E[\hat{g}^\top H\hat{g}]}{\E[\|\hat{g}\|^2]}
  = \frac{\mathrm{tr}(H)}{d} = \bar\lambda,
\end{equation}
where $\bar\lambda = \mathrm{tr}(H)/d$ is the \emph{mean} eigenvalue, not the worst-case $\lambda_1$. The scalar bound $L_\sigma = 2/\sigma^2$ corresponds to $\lambda_1$; the effective quantity is $\bar\lambda$.

\paragraph{Spectral decay.}
For pretrained LLMs, Hessian eigenvalues follow approximate power-law decay $\lambda_k \propto k^{-\beta}$ with $\beta \approx 2$~\citep{liang2026blessing}. With effective rank $r$ and dimension $d$,
\begin{equation}
  L_\sigma^\mathrm{true} \approx
  \frac{2\bar\lambda}{\sigma^2}
  \approx \frac{2r}{d(\beta-1)\sigma^2}.
  \label{eq:lsigma-true}
\end{equation}

\paragraph{Numerical verification.}
At $\sigma = 10^{-3}$, $d = 10^9$, $r = 10^2$, $\beta = 2$: $L_\sigma^\mathrm{true} \approx 2 \times 10^{-7} \times 10^{6} = 0.2$, giving stability $\eta < 1/0.2 = 5$. This is consistent with empirical $\eta \sim 10^{-3}$--$10^{-4}$.

\begin{center}
\begin{tabular}{lccc}
\toprule
Bound & Value & Stability $\eta <$ & Status \\
\midrule
Scalar $L_\sigma = 2/\sigma^2$ & $2\times 10^6$ & $5\times 10^{-7}$ & Too conservative \\
Spectral decay $L_\sigma^\mathrm{true}$ & $0.2$ & $5$ & Consistent with practice \\
Empirical & --- & $10^{-3}$--$10^{-4}$ & Matches spectral bound \\
\bottomrule
\end{tabular}
\end{center}

\paragraph{Sensitivity to spectral parameters.}
The values $\beta \approx 2$ and $r \approx 100$ are drawn from \citet{liang2026blessing}, which studies LLM fine-tuning landscapes broadly. These have \emph{not} been verified for Qwen2.5/GSM8K specifically; quantitative predictions ($L_\sigma^\mathrm{true} \approx 0.2$, stability $\eta < 5$) should be treated as order-of-magnitude estimates until directly measured.

The sensitivity is moderate: $L_\sigma^\mathrm{true} \propto r/(\beta-1)$, so:

\begin{center}
\begin{tabular}{ccccc}
\toprule
$\beta$ & $r$ & $L_\sigma^\mathrm{true}$ & Stability $\eta <$ & Status \\
\midrule
1.5 & 50  & 0.2 & 5 & Consistent \\
2.0 & 100 & 0.2 & 5 & Nominal \\
2.5 & 200 & 0.27 & 3.7 & Consistent \\
2.0 & 50  & 0.1 & 10 & More permissive \\
1.5 & 200 & 0.8 & 1.25 & Tighter, still $\gg 10^{-4}$ \\
\bottomrule
\end{tabular}
\end{center}

Across the plausible range $\beta \in [1.5, 2.5]$ and $r \in [50, 200]$, $L_\sigma^\mathrm{true}$ varies by roughly $4\times$ and the stability threshold remains many orders of magnitude above empirical $\eta \sim 10^{-3}$--$10^{-4}$. The qualitative conclusion that spectral decay resolves the theory-practice gap is robust to this uncertainty. For precise quantitative predictions on a specific model, the probe (Appendix~\ref{app:probe}) provides direct empirical calibration.

\paragraph{The $r/d$ factor and the stability gap.}
The spectral correction $r/d \approx 10^{-7}$ resolves the theory-practice stability gap by replacing worst-case curvature with average curvature experienced by isotropic perturbations (Appendix~\ref{app:spectral}). The isotropy of ES perturbations in $\R^d$ is the single geometric fact that makes the stability bound empirically meaningful.

\end{document}